\documentclass[journal]{IEEEtran}
\usepackage{color}

\usepackage{cite}
\usepackage{amsmath}
\usepackage{amsfonts}
\usepackage{amssymb}
\usepackage{amsthm}
\newtheorem{theorem}{Theorem}
\theoremstyle{definition}
\newtheorem{definition}{Definition}[section]
 
\theoremstyle{remark}

\ifCLASSINFOpdf
  \usepackage[pdftex]{graphicx}
  \DeclareGraphicsExtensions{.pdf,.jpeg,.jpg}
\else
\fi
\usepackage{algorithmic}
\ifCLASSOPTIONcompsoc
  \usepackage[caption=false,font=normalsize,labelfont=sf,textfont=sf]{subfig}
\else
  \usepackage[caption=false,font=footnotesize]{subfig}
\fi
\usepackage{url}
\begin{document}
\title{Vehicle Scheduling Problem}

\author{Mirmojtaba~Gharibi, 
        Steven~L.~Waslander,~\IEEEmembership{Senior Member,~IEEE,}
        and~Raouf~Boutaba,~\IEEEmembership{Fellow,~IEEE}% <-this % stops a space
\thanks{M. Gharibi and R. Boutaba are with D. Cheriton School
of Computer Science, University of Waterloo, Waterloo,
ON, N2L 3G1, Canada. Contacts: mgharibi@uwaterloo.ca and rboutaba@uwaterloo.ca respectively.}% <-this % stops a space
%\thanks{R. Boutaba is J. Doe and J. Doe are with Anonymous University.}% <-this % stops a space
\thanks{S. L. Waslander S. L. Waslander is with the Institute of Aerospace Studies, University of
Toronto, Toronto, ON M3H 5T6, Canada, stevenw@utias.utoronto.ca}
}


\markboth{}%
{Gharibi \MakeLowercase{\textit{et al.}}:  }



\maketitle

\begin{abstract}
We define a new problem called the Vehicle Scheduling Problem (VSP). The goal is to minimize an objective function, such as the number of tardy vehicles over a transportation network subject to maintaining safety distances, meeting hard deadlines, and maintaining speeds on each link between the allowed minimums and maximums. We prove VSP is an NP-hard problem for multiple objective functions that are commonly used in the context of job shop scheduling. With the number of tardy vehicles as the objective function, we formulate VSP in terms of a Mixed Integer Linear Programming (MIP) and design a heuristic algorithm. We analyze the complexity of our algorithm and compare the quality of the solutions to the optimal solution for the MIP formulation in the small cases. Our main motivation for defining VSP is the upcoming integration of Unmanned Aerial Vehicles (UAVs) into the airspace for which this novel scheduling framework is of paramount importance.
\end{abstract}


% \fi
\IEEEpeerreviewmaketitle

\section{Introduction}
 

\IEEEPARstart{V}{ehicle scheduling problem (VSP)} is a new scheduling problem that we define in this work. Many of the scheduling problems in the literature are motivated by real life applications. The motivation for defining this new scheduling problem is the impending integration of the Unmanned Aerial Vehicles (UAVs) in the airspace and a lack of framework for accommodating them in a scalable way. They will be used in a wide array of applications from search and rescue, to package delivery, traffic enforcement, infrastructure inspection and cinematography. This means in any city, there will be a high amount of congestion that needs to be managed to prevent mid-air collisions as well as to provide an efficient service\cite{Gha16}.
While our main motivation comes from the application of UAVs, our scheduling problem is generic and malleable enough to use in other areas as well.

On a high level, the problem we try to solve is as follows. We are given a path over a graph for each vehicle. Our goal is to minimize the number of tardy vehicles (or any other objective function) subject to the deadlines, minimum and maximum allowed speeds on the links, and separation time gap needed when entering the nodes (which play the role of intersections).

We first compare VSP with a large class of scheduling problems known as Job Shop Scheduling (JSP). The Job Shop Scheduling problem comes in many types which are motivated by the real life problems they strive to solve. Therefore, given the vast amount of literature on the subject, a first point of attack will be to model our problem in terms of a variant of JSP. In the classic Job Shop Scheduling problem, we are given a set of jobs, each composed of a chain of operations, and each operation can be performed on a specific machine from the set of all machines. A machine can only process one operation at a time at a specified processing time. At first sight, it seems the nodes in our graph can be simulated by machines and the separation time is analogous to the processing time. However, upon further inspection, it is not clear how to model the time it takes for a vehicle to reach from one node to another. Of course, this does not seem possible in a straightforward way using the classical type. However, even using variants of JSP with properties including BLOCKING, NO-WAIT, SETUP TIMES, etc. does not seem to represent our problem in an uncomplicated way. In JSP with BLOCKING, machines will hold up the operation and remain busy if the next machine is busy which corresponds to a lack of buffer between machines. In JSP with NO-WAIT, a task cannot wait between machines. In JSP with SETUP TIMES, there  will be a time delay to set up a new job on a machine. For a reference to these variants of JSP, look at e.g. \cite{Pin16}.

Some researchers have extended JSP and these more or less standardized variants to include transportation times between the machines. That is for a job to start executing in the next machine, it will be delayed by the transportation time between the previous machine and the new one. For instance, see \cite{Knu00} and \cite{Sch98}. However, in these cases, the transportation time is fixed whereas in VSP the velocity over a link can be chosen from a range. Furthermore, in some of these works, one or more robots are used for transportation with empty trips as necessary which again does not have a resemblance to our work \cite{Knu00},\cite{Hur05}.

scheduling in computer networks is another area with a vast literature. However, the underlying structure in computer networks is different. Firstly, the bottleneck are the router buffers (similar to the nodes in our case) and the packets spend most of their time in the routers whereas in our case vehicles spend only a minimal amount of time at a node and spend most of their times travelling on the links. Another difference is that packets might drop (i.e. vanish if needs be), but this is not an option for the vehicles.  Furthermore, the velocity of each packet during transmission on a link is constant and equal whereas vehicles can have different velocities over the links. These results in drastically different scheduling algorithms and policies which makes it difficult to use in VSP \cite{Pet07}.

For the application of scheduling of Unmanned Aerial Vehicles, we are aware of one work \cite{Yao19}. However, the authors formulate the problem as a joint optimization of routing and scheduling whereas in our case we are only interested in the scheduling aspect. This is a valid problem on its own as for various reasons, as the operator, we might not be authorized to make routing decisions. For instance, the government might restrict UAV flights for a company to only certain paths. 

In the context of air traffic management, the problems are formulated differently and the algorithms being used are not directly applicable to our problem. To be more precise, in the context of air traffic management, the air space is separated into sectors that are basically a volume of airspace and the goal is to avoid over loading each sector by means of either postponing the flight or changing the calculated route for a flight \cite{All12, Nol10}. Therefore the underlying graph structure in our problem is different.

An area of research that has some similarities is the train scheduling literature. There are similarities and differences between our approach. Similar to us, some of the train scheduling problems use a Mixed Integer Programming (MIP) model and a similar idea of treating arrival Time Stamps at a segment as a modelling variable. But most of these models are designed for a single main line with multiple short segments attached to the main line where trains can effectively park and let other trains take over before continuing their travel on the main line. There will be different set of constraints as well involved such as assignment of locomotives and crews which do not have a counterpart for unmanned aerial vehicles\cite{Fan15, Cor98}. 

\subsection{Our contribution}
Our contribution can be summarized as follows.

We define a new scheduling problem called Vehicle Scheduling Problem and formulate it in terms of a mixed integer linear programming. Our model has applications among other things to movements of autonomous vehicles over a transportation network; especially autonomous unmanned aerial vehicles. The model is versatile in that we can model vehicles with various minimum and maximum speeds as well as upperbounds on deadlines. Furthermore, we can adjust the safety gap as needed per pair of vehicles for each intersections (that is nodes of the graph).

We then proceed to show the NP-hardness of VSP for all commonly used objective functions in the context of job shop scheduling problems. These include minimizing
\begin{itemize}
\item Makespan: The time the last vehicle exits the graph.
\item Total (weighted) completion time: Total or (equivalently) the average travel time with potentially different weights for different vehicles.
\item Maximum lateness: The maximum (positive or negative valued) difference between deadline and trip completion of all vehicles.
\item Total tardiness: The total time past deadlines
\item (weighted) number of tardy vehicles: The number of vehicles that missed their deadlines.
\end{itemize}

It is possible to provide an MIP of VSP for all these objective function. To demonstrate that, we pick the objective function of number of tardy vehicles and give an MIP formulation of VSP. To deal with the computational complexity of VSP in this case, we devise a heuristic algorithm in the case where all trips are requested at the same time. We analyze the complexity of our algorithm and compare the solutions yielded from our algorithm to the optimal solution to MIP for a few random instances with a small number of vehicles. We also compare these results to a baseline algorithm that we designed.

Finally, we conclude with a discussion of the shortcomings of our algorithm such as sensitivity to noise as well as ideas for future improvement.

\section{Problem Definition}
We first give an informal definition of VSP and then proceed to define it in a more rigorous way.

In VSP, we are given a number of vehicles each of which request to make a trip at some point in time. Trips take place over a transportation network which is abstracted away as a graph. Each trip consists of a sequence of edges on this graph. We can have unlimited vehicles on each edge travelling at the same time, but when any two vehicles reach a vertex (i.e. intersection), they must be separated by some separation time constant. Therefore, to minimize some objective function (e.g. number of tardy vehicles) we have to schedule  vehicles in a way that minimizes the congestion or prioritizes certain vehicles. The velocity of each vehicle on each link must be between a minimum and a maximum velocity. Given the length of the link, one can instead find the minimum and maximum allowed time on the link. In the formal definition, instead of formulating in terms of velocities, we use these time constants. Each vehicle has to meet its trip completion hard deadline. The goal is to assign an arrival time for each vehicle at each vertex on their way (i.e. the schedule) such that the objective function is minimized.

Now, we define VSP in a more rigorous way.
\begin{definition}[Vehicle Scheduling Problem (VSP)]\label{Def VSP}
An instance of VSP is a $9$-tuple $\left(G,W,\tau_{min},\tau_{max},\rho, d, d', S,f\right)$ as follows. Graph 
$G=(V,E)$ is a directed connected graph that represents the traffic network. 
The set $W=\{W_1,W_2,\cdots, W_n\}$ is a set of directed walks on $G$ and $W_j=\left(W_j^1,W_j^2,\cdots,W_j^{q_j}\right) \in V^{q_j}$ is the sequence of vertices vehicle $j$ will visit. 
Time constants $\tau_{min}$ and $\tau_{max}$ are the set of minimum and maximum allowed times for a vehicle to reach its next vertex in the walk. Accordingly, for each vehicle $j$, 
$\tau_{min,j}$ takes the form $\tau_{min,j}=\left(\tau_{min,j}^1, \tau_{min,j}^2,\cdots,\tau_{min,j}^{q_j-1}\right)$. In a similar way, $\tau_{max,j}$ is defined.
The $n$-tuple $\rho \in \mathbb{R}^n$ denotes the trip request times and hard deadline $d'\in \mathbb{R}^n$ is the maximum allowed delays for completing the walks. Soft deadline $d\in \mathbb{R}^n$ is the maximum allowed delays after which depending on the objective function, a penalty will incur.
The set $S$ is a set of elements $s_{j_1,j_2}^{i_1,i_2}\in \mathbb{R}$ that specify the time separation between distinct vehicles $j_1$ and $j_2$ when performing the $i_1$'th and $i_2$'th step of their walks, respectively. Each element $s_{j_1,j_2}^{i_1,i_2}$ is defined only if vertices $W_{j_1}^{i_1}$ and $W_{j_2}^{i_2}$ are identical.

Subject to the following constraints, the goal is to find a set $t=\{t_1,t_2,\cdots,t_n\}$ where for each vehicle $j$, $t_j=\left(t_j^1,t_j^2,\cdots,t_j^{q_j}\right)$ assigns to each vertex $W_j^i$ an arrival time stamp while the objective function $f:\mathbb{R}^q \mapsto \mathbb{R}$ is minimized. Note that we define $q={\sum_{1\leq j\leq n} q_j}$ and $f$ takes the arrival time stamps $t_j^i$ as the input.
\\Constraints:
\\Trip request time, trip continuity, and hard deadline: 
\begin{equation}
\rho_j \leq t_j^1\leq t_j^2\leq\cdots\leq t_j^{q_j}\leq  d'_{j}
\end{equation}
Minimum and maximum allowable link travel time:
\begin{equation}
\tau_{min,j}^{i} \leq t_j^{i+1}-t_j^i \leq \tau_{max,j}^{i}
\end{equation}
Separation enforcement:
\begin{equation}\label{Eq Sep Abs}
|t_{j_1}^{i_1}-t_{j_2}^{i_2}| \geq s_{j_1,j_2}^{i_1,i_2}
\end{equation}
\end{definition}

In the definition above, only the vertices are included in the description of the walk and not the edges. In our definition, since edges have unlimited capacity, we do not concern ourselves with the particular edge that is chosen. However, in real life application, the length of a link and the minimum and maximum velocity among other factors might limit us to a particular edge.

\section{NP-hardness proof}
In the general case, VSP is NP-hard. 
We prove this for various objective functions.
First we define the GENERALIZED K-MACHINE UNIT JOB SHOP PROBLEM. For the cases $K=3$ and $K=2$, these problems are already defined in \cite{Len79} and \cite{Kra00} (with no release dates, deadlines, or no-wait condition), respectively. Since our goal is to reduce this problem to VSP, we do not use the more or less standard notation for JSP since it makes the description of the reduction difficult.

\begin{definition}[GENERALIZED K-MACHINE UNIT JOB SHOP PROBLEM: optimization version] 
An instance of GENERALIZED K-MACHINE UNIT JOB SHOP PROBLEM is a $6$-tuple $\left(M, J, r, \delta,\theta, f\right)$ where $M$ is a set of machines $M=\{M_1, M_2, \cdots, M_K\}$, and $J$ is a set of jobs $\{J_1,J_2,\dots,J_n\}$. Each job $J_j=\left(M_j^1,M_j^2,\cdots, M_j^{q_j}\right)\in M^{q_j}$ denotes a sequence of machines that will each process job $J_j$ in order, for a duration of one time unit. 
The $n$-tuple $r=\left(r_1,r_2,\cdots,r_n\right)$ denotes the release dates and $\delta=\left(\delta_1,\delta_2,\cdots,\delta_n\right)$ denotes the deadlines corresponding to each job, respectively. Additionally, $\theta$ is a Boolean parameter that is $true$ if jobs cannot wait between machines, otherwise $false$.

Subject to the following constraints, the goal is to find a set $x=\{x_1,x_2,\cdots,x_n\}$ where for each job $j$, $x_j=\left(x_j^1,x_j^2,\cdots,x_j^{q_j}\right)$ assigns to each operation $J_j^i$ a starting time on their associated machine  while the objective function $f:\mathbb{R}^q \mapsto \mathbb{R}$ is minimized. Note that we define $q={\sum_{1\leq j\leq n} q_j}$ and $f$ takes the arrival time stamps $x_j^i$ as the input.
\\Constraints: 
\\
-- Operations done in order, i.e.  
$x_j^i+1\leq x_j^{i+1}$.
\\
-- Any machine can process only one operation at a time, i.e. for any distinct $j_1,j_2$, if there exists $i_1,i_2$ such that $J_{j_1}^{i_1}=J_{j_2}^{i_2}$, then $\left|x_{j_1}^{i_1}-x_{j_2}^{i_2}\right|\geq 1$.
\\
-- No job scheduled before its release date, i.e.
$x_j^1\geq r_j$.
\\
-- If deadlines are hard deadlines, then meeting deadlines, i.e.
$x_j^{q_j}\leq \delta_j$.
\\
-- No-wait condition, i.e. 
if $\theta$ is $true$, for any $j$ and any $i$ with $1\leq i\leq q_j-1$, we have $x_j^i+1 = x_j^{i+1}$.

\end{definition}

We first show there is an efficient reduction from GENERALIZED K-MACHINE UNIT JOB SHOP PROBLEM with an arbitrary objective function to VSP.
\begin{theorem}\label{thm-reduction}
GENERALIZED K-MACHINE UNIT JOB SHOP PROBLEM has a polynomial time reduction to VSP.
\end{theorem}
\begin{proof}
The input to GENERALIZED K-MACHINE UNIT JOB SHOP is a $6$-tuple $(M,J,r,\delta,\theta,f)$. The input to VSP is a $9$-tuple $(G,W,\tau_{min},\tau_{max},\rho,d,d',S,f)$. We generate an input to the latter, for every input to the former.
We set $G$ with vertices $V_1,V_2,\cdots,V_K$ to be a complete digraph with $K$ vertices.
We establish a one-to-one correspondence between the jobs in $J$ and walks in $W$ as follows. Assuming $|J|=n$ and job $J_j$ has a length of $q_j$;
for $1\leq j\leq n$ and the job $J_j=(M_j^1,M_j^2,\cdots,M_j^{q_j})$, we create a corresponding walk $W_j=(W_j^1,W_j^2,\cdots,W_j^{q_j})$ where $W_j^i=V_{j'}$ if and only if $M_j^i=M_{j'}$.
The sequences $\rho$, $d$, and $d'$ will have a length of $n$ and for any $j$, both $\tau_{min,j}$ and $\tau_{max,j}$ have length $q_j-1$.
Next, we set $\tau_{min,j}=\left(1,1,\cdots,1\right)$.
If $\theta$ is $false$ (i.e. jobs can wait), we set $\tau_{max,j}=\left(+\infty,+\infty,\cdots,+\infty\right)$, otherwise if $\theta$ is $true$ (no-wait), $\tau_{max}=\left(1,1,\cdots,1\right)$.
We set $\rho=r$. For the deadlines, if $\delta$ is a hard deadline, we set $d=d'=\delta$. However, if $\delta$ is a soft deadline, we set $d=\delta$ and $d'=+\infty$.
Lastly, we will use the same objective function $f$ in the reduction with arguments changing from $x_j^i$ to $t_j^i$.
This completes the conversion of the inputs. Now, to translate the solution for VSP to GENERALIZED K-MACHINE UNIT JOB SHOP is straightforward.  To convert the outputs, for any arrival time stamps $t_j^i$ for $1\leq j\leq n$, we let $x_j^i=t_j^i$, and with this, the reduction is complete.
\end{proof}

We prove NP-hardness for all commonly used objective functions in the context of job shop scheduling \cite{Pin16}.

\begin{theorem}
VSP with any of the following objective functions is NP-hard.

Minimizing:
\begin{itemize}
\item Makespan ($C_{max}$)
\item Total completion time ($\sum C_j$)
\item Total weighted completion time ($\sum w_j C_j$)
\item Maximum lateness (lateness can be positive, 0, or negative for each vehicle) ($L_{max}$)
\item Total tardiness (0 or positive for each vehicle) ($\sum T_j$)
\item Weighted number of tardy (or late) vehicles ($\sum w_j U_j$)
\item Number of tardy (or late) vehicles ($\sum U_j$)
\end{itemize}
\end{theorem}
\begin{proof}

For any of these objective functions, we refer to a special case of GENERALIZED  K-MACHINE  UNIT  JOB SHOP  PROBLEM that is proven NP-hard in the literature. Then, by applying Theorem \ref{thm-reduction}, the NP-hardness of VSP is established for that objective function. In the following, the objective functions related to tardy vehicles correspond to tardy jobs in the context of JSP problem.

\textit{Makespan}: GENERALIZED K-MACHINE UNIT JOB SHOP is NP-hard according to \cite{Len79}, where $K=3$ and all jobs are released at time $0$ ($r_j=0$) and there is no deadline ($d_j = +\infty$) and the jobs can wait ($\theta=false$). This problem is referred to as 3-MACHINE UNIT JOB SHOP. 

 \textit{Total tardiness}:  GENERALIZED  K-MACHINE  UNIT  JOB SHOP  PROBLEM with $K=2$ and the same setting as in the Makespan case above is NP-hard. This problem is referred to as 2-MACHINE UNIT JOB SHOP \cite{Kra00, Tim98}. 

\textit{Total (weighted) completion time}, \textit{(Weighted) number of tardy vehicles}: Using the same setting as above for 2-MACHINE UNIT JOB SHOP PROBLEM with release dates, this problem is NP-hard. Note that without the release dates, this problem  is polynomially solvable \cite{Kra00,Kub96,Tim98}. 

\textit{Maximum lateness}: For similar setting as in the case of Total Completion Time, 2-MACHINE UNIT JOB SHOP PROBLEM where jobs cannot wait between two machines ($\theta=true$) is strongly NP-hard \cite{Tim98}. 
\end{proof}

\section{Exact MIP formulation}
In the next section, we give an MIP formulation of the VSP where to demonstrate, we use the number of tardy vehicles as our objective function. It is possible to develop MIP formulation for the other objective functions as well by perhaps introducing more variables and doing the necessary minor adjustments.
\subsection{Notations}
We use the same notation from VSP problem and some additional notations as follows.
\begin{itemize}
\item Variables $P_{j_1,j_2}^{i_1,i_2}$, $N_{j_1,j_2}^{i_1,i_2}$, decision binary variable $b_{j_1,j_2}^{i_1,i_2}$, and a large fixed number $M_{j_1,j_2}^{i_1,i_2}$ will be used to convert the absolute value constraints to mixed integer linear constraints. The decision binary variable will effectively decide between each pair of vehicles with a conflicting node, which one will have the right of way. 
\item Variable $l_j$ will designate a late vehicle, $X_j$ will designate how late vehicle $j$ is and a large fixed number $M_{j}$ will be used to convert the constraints of form $X_j=\max(0,d'_j-t_j^{q_j})$ into mixed integer linear constraints. We add variable $d_j$ to designate a soft deadline that a vehicle will strive to meet. On the other hand, the hard deadline $d'_j$ must be met to have a feasible solution.
\end{itemize}

\subsection{Objective function}
We use the total number of tardy vehicles as our objective function.
\begin{equation}\label{eq-objective-function}
\min \sum_{1\leq j \leq n}l_j.
\end{equation}

\subsection{Constraints}
We use the same constraints as in Definition \ref{Def VSP} except for the constraints that follow. The first constraint is the separation enforcement constraint of Eq. \ref{Eq Sep Abs}
where through the standard techniques, the absolute value constraint can be converted to a set of three mixed integer linear constraints.

\begin{equation}
t_{j_1}^{i_1}-t_{j_2}^{i_2} = P_{j_1,j_2}^{i_1,i_2} - N_{j_1,j_2}^{i_1,i_2}
\end{equation}
\begin{equation}
b_{j_1,j_2}^{i_1,i_2} \cdot s_{j_1,j_2}^{i_1,i_2} \leq P_{j_1,j_2}^{i_1,i_2} \leq b_{j_1,j_2}^{i_1,i_2} \cdot M_{j_1,j_2}^{i_1,i_2}
\end{equation}
\begin{equation}
\left(1-b_{j_1,j_2}^{i_1,i_2}\right) \cdot s_{j_1,j_2}^{i_1,i_2} \leq N_{j_1,j_2}^{i_1,i_2} \leq \left(1-b_{j_1,j_2}^{i_1,i_2}\right) \cdot M_{j_1,j_2}^{i_1,i_2}
\end{equation}

Furthermore, we convert the late vehicle constraints into the acceptable form for a linear integer programming instance using familiar techniques as follows.
\begin{equation}
d_j - t_j^{q_j}\leq X_j
\end{equation}
\begin{equation}
0\leq X_j
\end{equation}
\begin{equation}
X_j\leq M_j\cdot \left(1-l_j\right)
\end{equation}
\begin{equation}
X_j\leq d_j - t_j^{q_j} + M_j\cdot l_j
\end{equation}

\section{Scheduling algorithms}
\subsection{Baseline algorithm: PROXIMITY }
We introduce this baseline algorithm with the purpose of comparing its result to the result from the heuristic algorithm introduced in the next section. Our goal is to establish the superiority of the latter algorithm. 

The algorithm PROXIMITY  resembles the decentralized heuristics used by car drivers in the real world and is as follows. Each vehicle on a link attached to an intersection will access that intersection at the earliest possible time, if and only if it is the closest (in time) vehicle to cross that intersection. An implementation of this algorithm is shown in Fig. \ref{algo:psuedocode1} and Fig. \ref{algo:psuedocode2} with $mode$ variable set accordingly.

\subsection{Heuristic algorithm: DEADLINE \& PROXIMITY }
Our heuristic algorithm is as follows. The algorithm returns the best solution from three independent subroutines. 
\begin{itemize}
    \item PROXIMITY : It is the baseline algorithm introduced earlier.
    \item ABS DEADLINE \& PROXIMITY: It is similar to PROXIMITY  with the difference that among vehicles of equal (time) distance, the ones with a lower so called delay slack are prioritized over those with higher values. We calculate the delay slack for a vehicle as the difference between its deadline and the shortest trip time possible. However, if no delay slack is left; that is the calculated delay slack is a negative quantity, we give these vehicles the lowest priorities and among themselves, they will cross the intersection in arbitrary orders. 
    \item REL DEADLINE \& PROXIMITY: It is similar to the previous subroutine except that the delay slack is divided by the number of intersections left in the trip.
\end{itemize}
An implementation for each of these subroutines is shown in Fig. \ref{algo:psuedocode1} and Fig. \ref{algo:psuedocode2} with $mode$ variable set accordingly.

\begin{figure}[!htb]
\centering
 \begin{algorithmic}[1]
 \renewcommand{\algorithmicrequire}{\textbf{Input:}}
 \renewcommand{\algorithmicensure}{\textbf{Output:}}
 \REQUIRE paths for vehicles, $mode$
 \ENSURE  time stamps for each node
 \STATE $timeStampSequence = [ ]$ \textit{//empty list} 
 \STATE $allVehiclesList =$ list of all vehicles
 \STATE sort $allVehiclesLis$t ascendingly using $SortingKey$ for comparison
    \FORALL {$vehicle$ in $allVehiclesList$}
        \STATE assign earliest possible time stamp to the $vehicle$'s first node and update data structures
        \STATE if does not exist, add time stamp to $timeStampSequence$ in a sorted increasing order
    \ENDFOR
\WHILE{$length(timeStampSequence)>0$}
    \STATE $t=timeStampSequence[0]$
    \STATE $curVehList=$ list of all unfinished vehicles with time stamp $t$ for their current node
    \FORALL{$v$ in $curVehList$}
        \STATE $N=$ next node in path of $v$
        \STATE $conflictVehList=$ list of all vehicles with $N$ in their trip
        \STATE $curVehNextNList=$ list of all vehicles with next node $N$ and arrival time stamp $t$ on their current node
        \STATE sort $curVehNextNList$ ascendingly using $sortingKey$ for comparison
        \FORALL{$v2$ in $curVehNextNList$}
            \IF{$v2$ has no time stamp for $N$}
                \STATE assign the earliest time stamp for $N$ for $v2$ that satisfies separation constraints imposed by $conflictVehList$ and update data structures
            \ELSE
                \STATE continue
            \ENDIF
            \STATE if does not exist, add this time stamp to $timeStampSequence$ in a sorted increasing order
        \ENDFOR
    \ENDFOR
    \STATE delete $timeStampSequence[0]$
\ENDWHILE
\end{algorithmic}
 \caption{A subroutine to calculate the time stamps for vehicles. Depending on our input to the algorithm $mode$ variable, the time stamps are calculated using the baseline rule or our heuristic rule.}
\label{algo:psuedocode1}
\end{figure}

\begin{figure}[!htb]
\centering
 \begin{algorithmic}[1]
 \renewcommand{\algorithmicrequire}{\textbf{Input:}}
 \renewcommand{\algorithmicensure}{\textbf{Output:}}
\REQUIRE $vehicle$, $mode$
\ENSURE A two tuple to be used for sorting by first element and then the second element 
\STATE $first=$time distance to next node
\IF{$mode$ is PROXIMITY }
    \STATE $second=$ None
    \ELSIF{$mode$ is ABS DEADLINE \& PROXIMITY }
    \STATE $second=\max(0,$ remaining delay slack$)$
\ELSIF{$mode$ is REL DEADLINE \& PROXIMITY }
    \STATE $second=\max(0,$ delay slack $/$ number of remaining of nodes$)$
\ENDIF
\RETURN (first,second)
\end{algorithmic}
 \caption{The $SortingKey$ subroutine returns a 2-tuple for each vehicle that determines their local priority based on the chosen rule.}
\label{algo:psuedocode2}
\end{figure}

\subsection{Complexity}
The implementation of the three subroutines mentioned above are the same with minor differences in the  $SortingKey$ subroutine component in Fig. \ref{algo:psuedocode2}. This results for the time complexity of all these algorithms to be similar. This is true because with efficient implementation of the data structures, the $SortingKey$ has the same time complexity of $O(1)$ in all of the cases. Therefore to obtain the time complexity of all the three subroutines, we need only to obtain the time complexity of the subroutine in Fig. \ref{algo:psuedocode1}. 

The time complexity of the algorithm will depend on the actual implementation of the data structures which is as follows. In the implementation, we assume we populate these data structures at the point in the algorithm where $t$ is assigned which will take $O(1)$ time.
We keep the following hash tables:
\begin{itemize}
    \item Time stamps to list of vehicles: $H_1:t \mapsto V$
    \item 2-tuples of current node time stamps and next nodes to list of vehicles: $H_2: (t,N) \mapsto V$. 
    \item Nodes to the list of time stamps: $H_3: t \mapsto (t,N)$
\end{itemize}

The most time consuming operations is step 22 in Fig. \ref{algo:psuedocode1} which takes $O\left(q\right)$ where $q$ is the total number of time stamps. Assuming $D$ is the degree of the underlying graph, considering the outer loops, the total time spent on this operation is $O\left(q^2Dn\right)$ where $n$ is the number of vehicles.

If we make the further assumption that the length of each path is constant, the worst case complexity of both the baseline and the heuristic algorithm is simplified to $O\left(n^3\right)$.

\section{Numerical Results}

In this section, we compare the performance of our heuristic algorithm DEADLINE \& PROXIMITY  to our baseline algorithm  PROXIMITY. Additionally, for the smaller cases, We compare the results from both these algorithms to the exact solution obtained from solving the MIP. We consider three metrics in our comparisons, namely, the number of vehicles, the tightness of the soft deadlines, and the run time of the algorithms.

In our setup, we create uniformly at random pairs of source and destinations and calculate the shortest paths on the grid like graph $G$ as shown in Fig \ref{fig:Fig-GridGraph}. We use the following parameters.
\begin{itemize}
\item Separation gap $s_{j_1,j_2}^{i_1,i_2}=5$ for any $j_1,j_2,i_1,i_2$.
\item Minimum and maximum times on link $\tau_{min,j}^i=50$ and $\tau_{max,j}^i=+\infty$ for all $i,j$.
\item Hard deadlines $d'_{j}=2.2q_j\tau_{min,j}$ for each $j$.
\item Trip request time $\rho_{j}=0$ for each $j$.
\end{itemize} 

For our setup we used Gurobi optimization engine version 8.1.1 running on Microsoft Surface Pro 5 with 8GB RAM and 4 Intel(R) Core(TM) i5-7300U CPUs clocking at 2.60GHz. A Gurobi Python 3.6 binding was used to solve the MIP exactly. Also, we implemented the baseline and the heuristic algorithm in Python 3.6.

We create $20$ instances of VSP with four different vehicle counts; 25, 50, 75, and 100. For each instance, we experiment with varying levels of tightness of the soft deadline. The fraction of missed deadlines corresponding to each deadline value are reported in Fig. \ref{fig1}. The worst run time over all deadlines for each number of vehicles is reported in Fig. \ref{fig:chartRunTimes} for the baseline and the heuristic algorithms.

For 25 vehicles, Fig. \ref{1a} shows the results comparing the percentage of missed deadlines for our heuristic algorithm, the baseline, and the exact solution from MIP. Given the small difference between the result from the baseline algorithm and the exact solution, it is plausible to infer the baseline algorithm is in fact a very effective one in producing schedules with a few number of missed deadlines.

In the case of 25 vehicles, we are able to solve most of the MIP instances exactly in a reasonable amount of time (less than 1 hour). Since The MIP quickly becomes intractable with increasing the number of vehicles; the 25 vehicles is about the maximum number of vehicles that can be used in our test. The comparison between the run time of our heuristic algorithm to the MIP instance solved by Gurobi is reported in Fig. \ref{fig:HeuToMIP} on a logarithmic scale. In most cases, our algorithm is between 1 to 3 orders of magnitude faster, despite the fact that our heuristic algorithm is implemented in notoriously slow Python. 

Finally, Fig. \ref{fig:chartRunTimes} shows a comparison between the run time of our heuristic algorithm and the baseline algorithm.

\begin{figure}[!htb]
\centering
\includegraphics[trim=0 50 350 0, clip,width=2.5in]{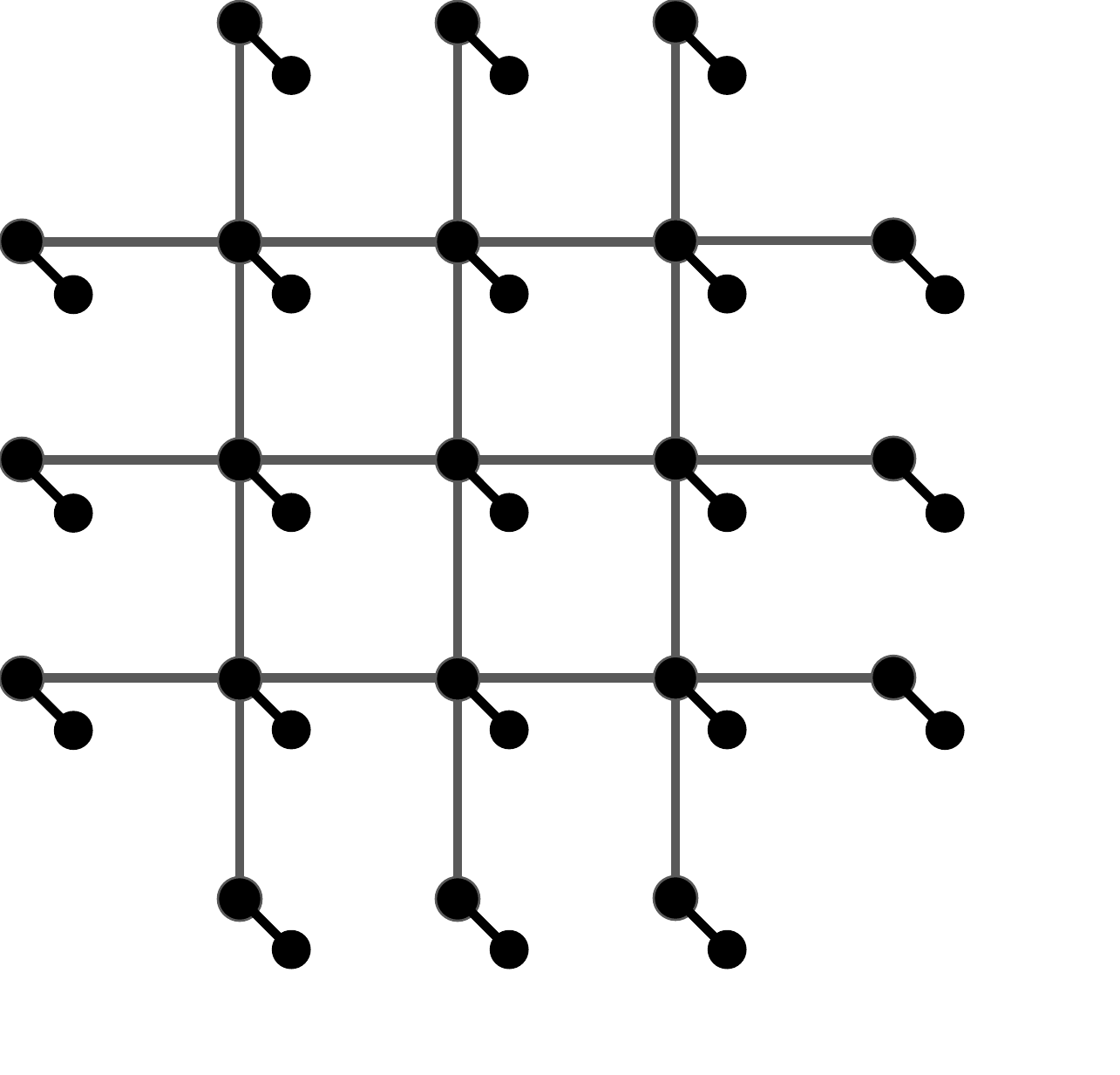}
%pdfcrop --margins '-60 -5 -10 -25' FigureCellularPDF.pdf FigureCellularCropPDF.pdf
% where an .eps filename suffix will be assumed under latex, 
% and a .pdf suffix will be assumed for pdflatex; or what has been declared
% via \DeclareGraphicsExtensions.
\caption{Underlying graph for simulation is a 5x5 grid like   graph. Each edge represent a bidirectional path with dedicated lanes for each direction. }
\label{fig:Fig-GridGraph}
\end{figure}

\begin{figure*} 
    \centering
  \subfloat[Simulation result for 25 vehicles\label{1a}]{%
       \includegraphics[width=3.5in]{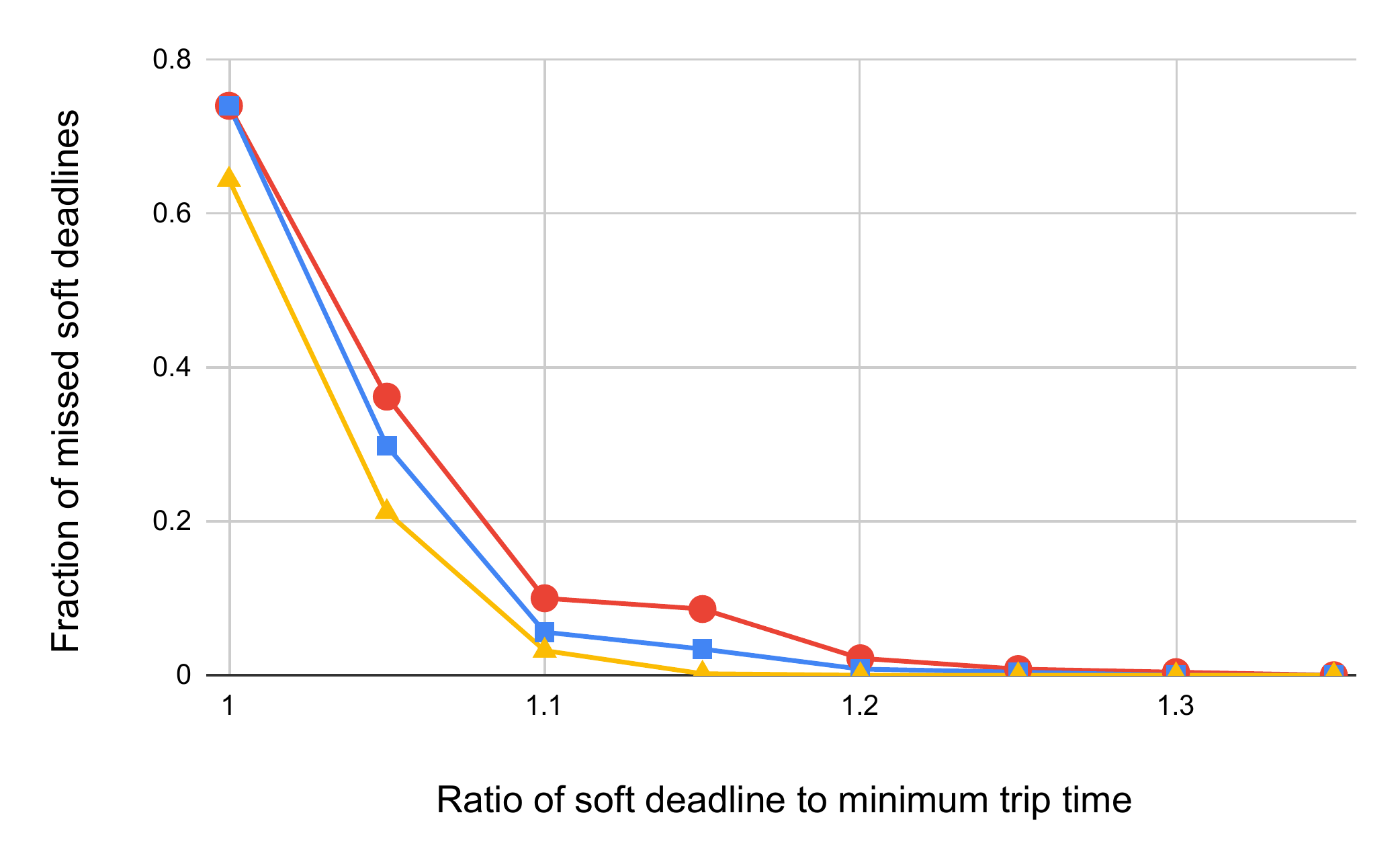}}
       \hfill
  \subfloat[Simulation result for 50 vehicles\label{1b}]{%
        \includegraphics[width=3.5in]{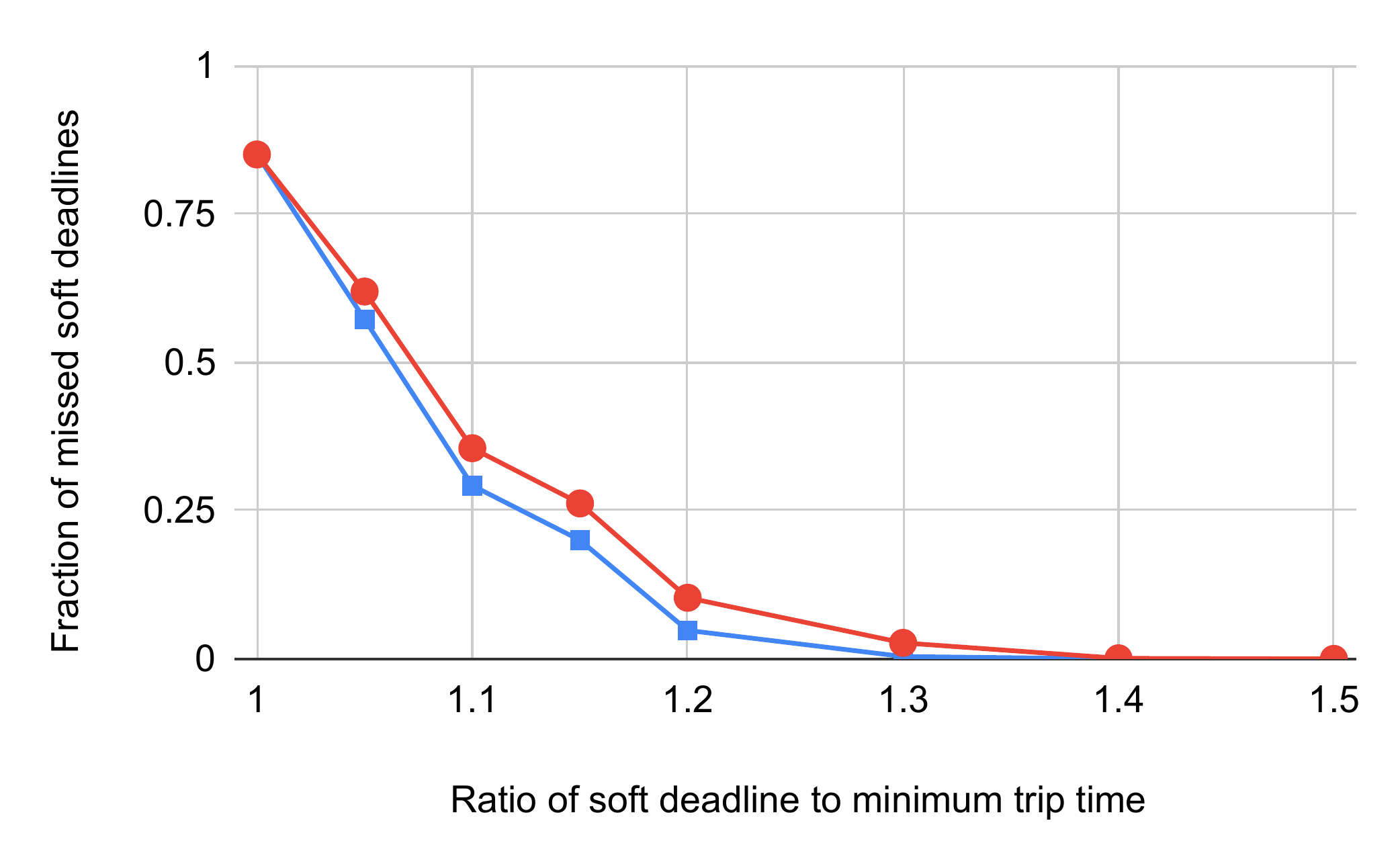}}
    \\
  \subfloat[Simulation result for 75 vehicles\label{1c}]{%
        \includegraphics[width=3.5in]{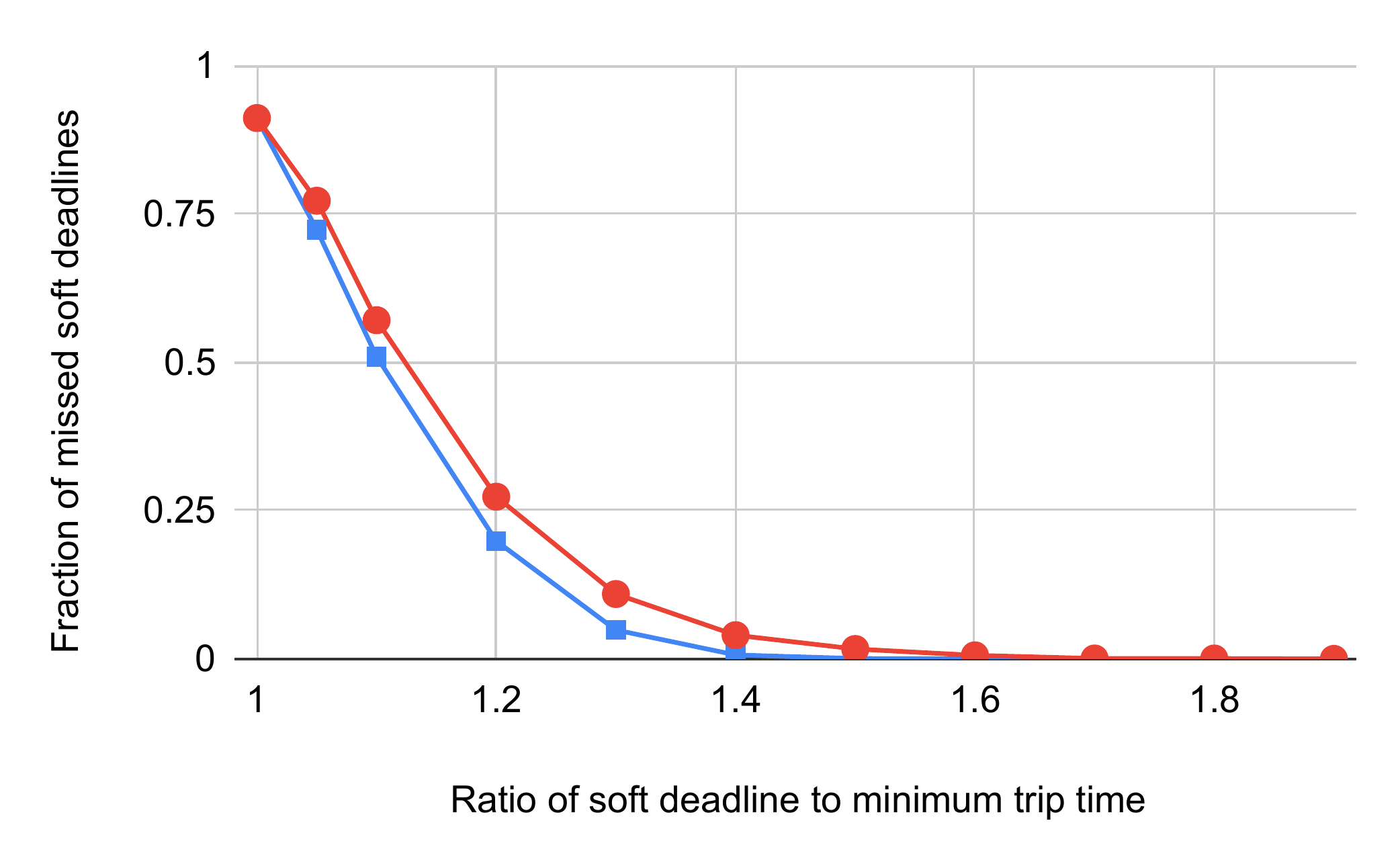}}
    \hfill
  \subfloat[Simulation result for 100 vehicles\label{1d}]{%
        \includegraphics[width=3.5in]{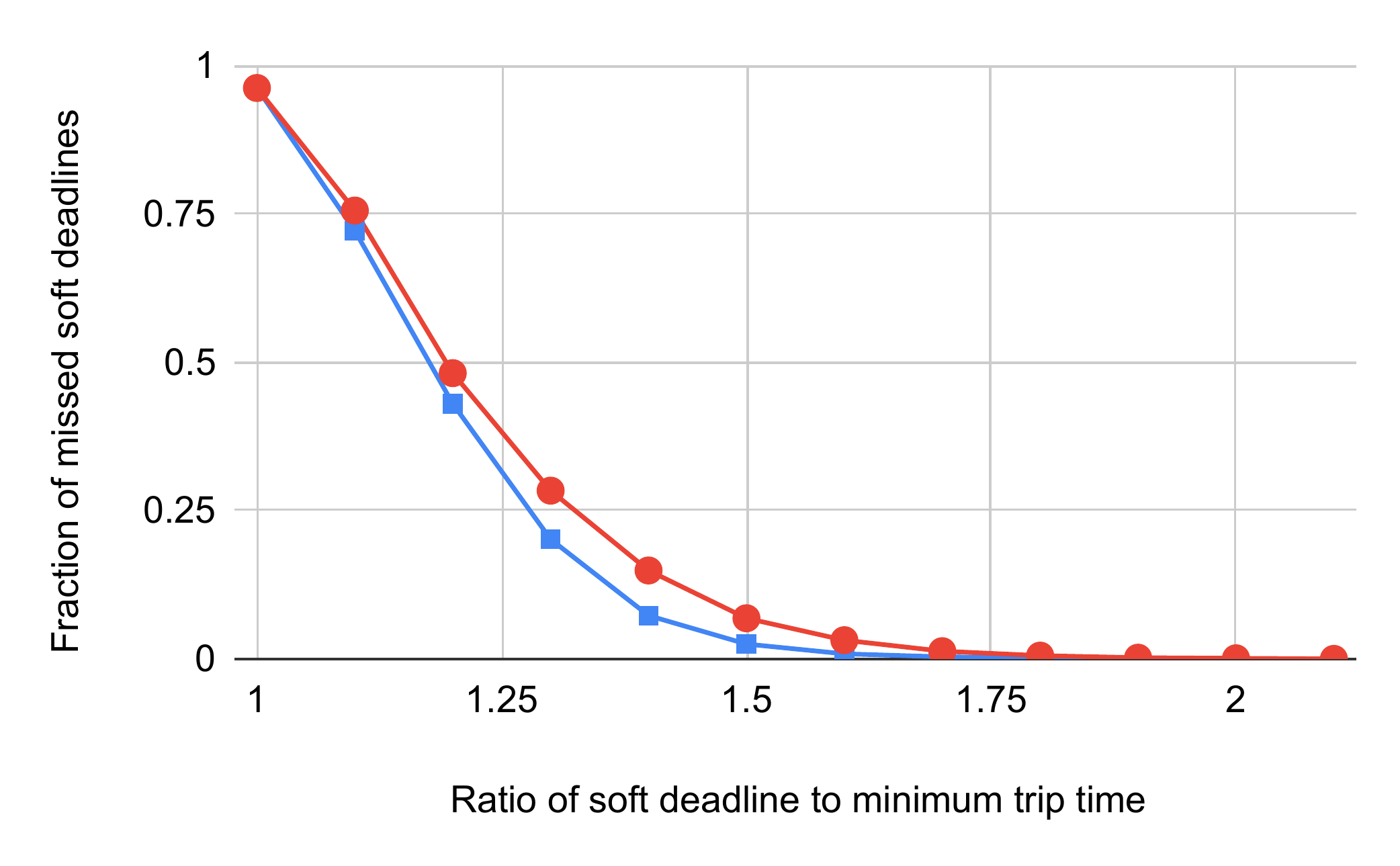}}
  \caption{(a), (b), (c), (d) Simulation results for respectively 25, 50, 75, and 100 vehicles with randomly generated trips comparing the solutions to the baseline algorithm (circle markers) versus our heuristic algorithm (square markers). For the case of 25 vehicles, we include also the exact MIP solution (triangle markers). The vertical axis is the fraction of tardy vehicles averaged over 20 random instances and the horizontal axis is the ratio of the set soft deadlines to the minimum congestion free trip times.}
  \label{fig1} 
\end{figure*}

\begin{figure}[!htb]
\centering
\includegraphics[trim=0 0 0 0, clip,width=3.5in]{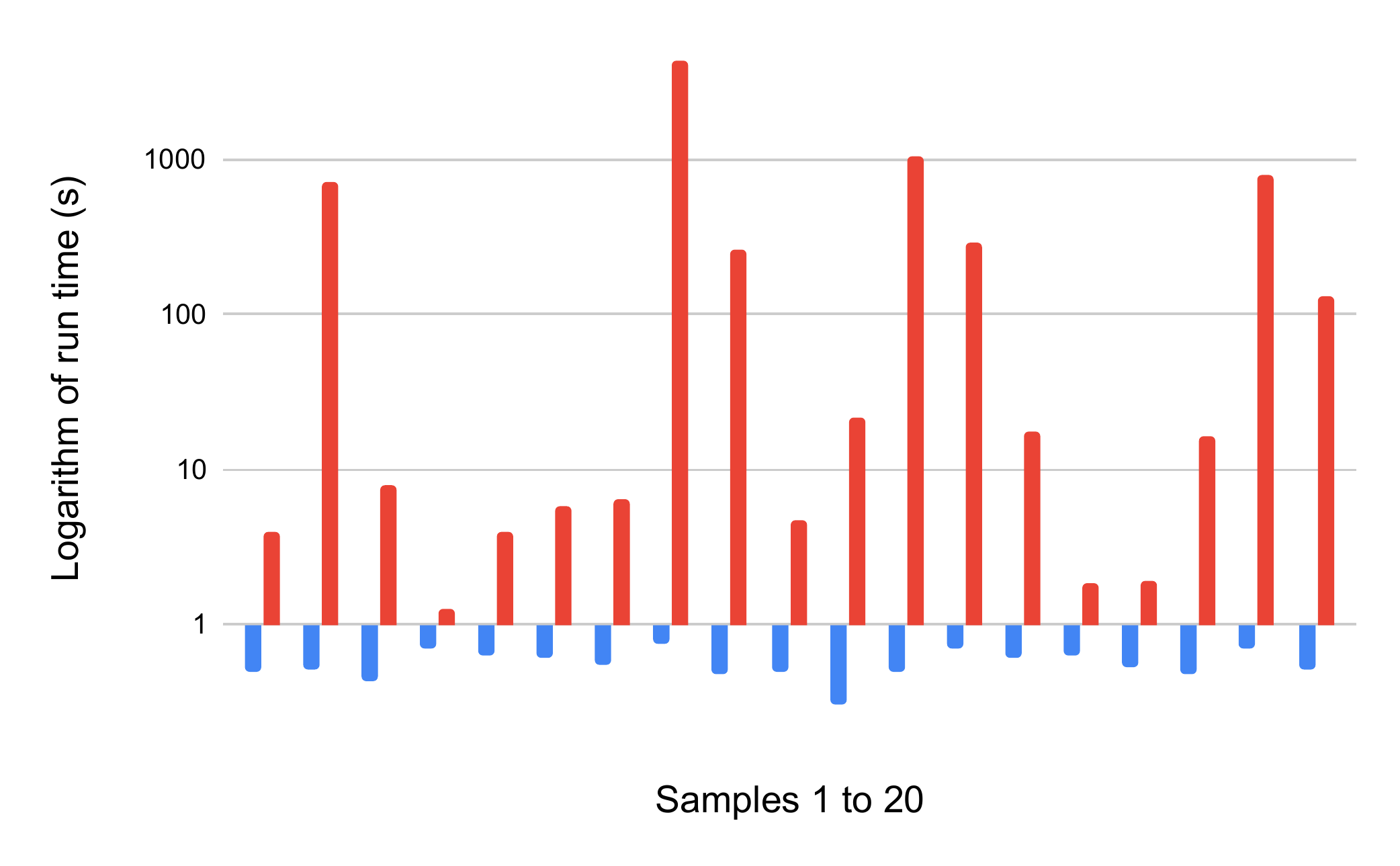}
%pdfcrop --margins '-60 -5 -10 -25' FigureCellularPDF.pdf FigureCellularCropPDF.pdf
% where an .eps filename suffix will be assumed under latex, 
% and a .pdf suffix will be assumed for pdflatex; or what has been declared
% via \DeclareGraphicsExtensions.
\caption{This figure shows the run time results of our heuristic algorithm compared to the MIP solver in logarithmic scale for 20 random instances. Each value represents the average of worst run time among various levels of tightness of soft deadlines. These deadline values are similar to those in Fig. \ref{1a}. 
Each of these instances utilize 25 vehicles. The red bars which all happen to have a value greater than 1 are the run time from MIP solver and the blue bars under the grid line for 1 are the run time from our heuristic algorithm. As can be seen, in most cases, our algorithm implemented in Python is between 1 to 3 orders of magnitude faster.}
\label{fig:HeuToMIP}
\end{figure}

\begin{figure}[!htb]
\centering
\includegraphics[trim=0 0 0 0, clip,width=3.5in]{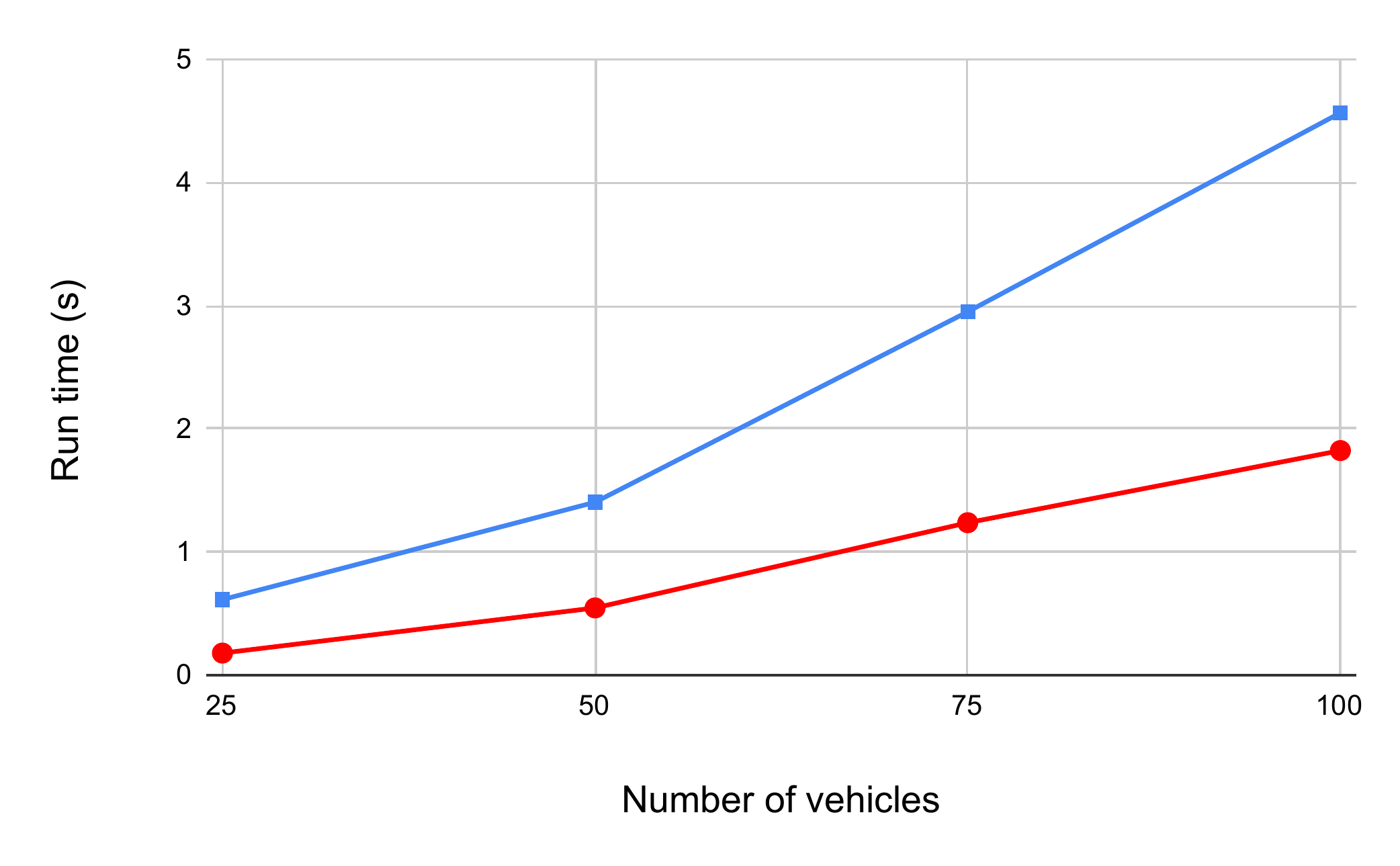}
%pdfcrop --margins '-60 -5 -10 -25' FigureCellularPDF.pdf FigureCellularCropPDF.pdf
% where an .eps filename suffix will be assumed under latex, 
% and a .pdf suffix will be assumed for pdflatex; or what has been declared
% via \DeclareGraphicsExtensions.
\caption{This figure shows the average taken over 20 samples of the worst run time among varying levels of soft deadlines. The value of soft deadlines in our test are the same as those appearing in Fig. \ref{fig1}. The chart shows how the run time is affected by increasing the number of vehicles from 25 to 100. The square markers represent results from our heuristic algorithm while the circle markers represent the baseline results.}
\label{fig:chartRunTimes}
\end{figure}

%\subsection{Implementation}
%We use a Microsoft Surface Pro 5 with Intel Core i5-7300U CPU at 2.60GHz (4 Cores) and 8GBs of RAM. It runs an Ubuntu 18.04.1 LTS Linux operating system. We use Gurobi Version 8.1.1's Python 3 binding for constructing the model and optimization for both MIP and LP instances. The shortest path algorithm is by Networkx version 2.3 Graph library for Python. In our comparison, the exact same instance of the problem is fed to both algorithms.

\section{Discussion and future work}
In this work, we assumed the route for each vehicle is fixed. There are scenarios where this might make sense for example when some UAV companies are allocated various paths in the airspace by the government and all the scheduling must be performed over these preallocated paths. However, in absence of these kind of limitations, an interesting problem to consider for future research is the joint optimization of routing and scheduling.

Another simplified aspect of the problem is the assumption that each link has unlimited capacity for holding vehicles. This is not entirely without merit as the in-flow and out-flow are bounded by various constraints. However, it might be plausible to set a direct capacity limit on a link, especially when the minimum permitted velocity on a link is $0$. 

With minor modifications, it is possible to expand the DEADLINE \& PROXIMITY algorithm and the baseline PROXIMITY to the dynamic VSP to make it more applicable to the real word. It is interesting to see how these two algorithms will compare in that case.

In our numerical result section, we demonstrated for some special cases the result from our heuristic algorithm. It will be interesting to explore more cases, such as when the trip request times are arbitrary and the minimum speed for vehicles can be more than 0.

A limitation of our work is that in our heuristics algorithm DEADLINE \& PROXIMITY, the access rule is based on first the proximity of the vehicles to the intersections and only then the deadlines are used as a tie breaking mechanism. In our setup, the time stamps from the created schedules were all a multiple of an integer (in our case, 5), allowing the tie breaking rule a chance to have an influence. In practice, the tie breaking will likely never be used as the odds of having two vehicles at the exact same distance (up to our measurement digital precision) is close to 0. In other words, in presence of noise, etc., our heuristic algorithm degrades to the baseline algorithm. A simple fix is to use a window of a certain size inside which the deadline is the deciding rule. However, to determine the optimal window size will be the subject of a future work.

Another question with regard to both the baseline and the heuristic algorithms is how irregular sizes for the links will affect the quality of the created schedules. For example, a vehicle might be an intersection away from an intersection of interest, and only a very short distance away. But it will get a lower priority compared to a further away vehicle which is already on a connected link to the intersection. It is an interesting line of research to pursue expanding the pool of eligible vehicles in deciding the right of way to include vehicles such as the one in the example above.

Finally, there is an approximation algorithm that is used for Job Shop Scheduling known as Shifting Bottleneck as first appeared in the seminal work \cite{Ada88}. The goal was to minimize the makespan (completion time of the last job to finish). The algorithm works by sequencing each machine as a one machine optimization problem which can be solved efficiently. At any point, a list of sequenced machines and a list of unsequenced machines exists. Based on some criteria about which machine is the biggest next bottlenecks, the unsequenced machines are ranked and the machine with the highest rank is chosen to be sequenced next. Based on the results, the already sequenced machines are resequenced one by one till no further improvements can be found in their schedule. Again, based on the outcome of this step, and according to the bottleneck criteria the next machine is chosen and the process is repeated till all the machines are sequenced. In our context, it is an interesting venue for research to see if any algorithm with a similar idea of detecting the bottleneck vehicles or the bottleneck nodes or even a set of bottleneck neighboring nodes (similar to a zone) can be used to give better schedules.

\section{Conclusion}
In this work we introduced a new scheduling problem called Vehicle Scheduling Problem. Given a path between a pair of source and destination for each vehicle over a graph, the goal is to minimize some objective function such as the number of tardy vehicles (i.e. missed deadlines) subject to various constraints as follows. This includes, maintaining a safety time gap at conflicting nodes and meeting hard deadlines after trips are requested by vehicles. Furthermore, each vehicle is required to maintain its speed in an allowable range over any link. We established the NP-hardness of VSP for all commonly used objective functions in the context of JSP. Then, we formulate this problem in terms of a Mixed Integer Linear Programming where the chosen objective function is the number of tardy vehicles. For the case of simultaneous trip requests, we then devised a heuristic algorithm based on giving priority to vehicles closer to an intersection and with less slack time left. We also devised a baseline algorithm that mimics to some extent the real world traffic, i.e. vehicles closer to an intersection will get there first. We then performed numerical experiments on random instances of the problem over a grid like graph to compare the obtained objective value from the exact solution to the MIP formulation as well as the baseline algorithm and our algorithm.

% or


% \appendices

% \section{}

\section*{Acknowledgment}
We would like to thank Stephen Smith for fruitful discussions on the subject.

\ifCLASSOPTIONcaptionsoff
  \newpage
\fi



\bibliographystyle{IEEEtran}
\bibliography{IEEEabrv,refs}

%\begin{thebibliography}{1}

%\bibitem{Rap13}
%A.~Raptopoulos, \emph{No roads? There is a drone for that}, 2013. [Online]. \hskip 1em plus
%  0.5em minus 0.4em\relax Available: \url{https://www.ted.com/talks/andreas_raptopoulos_no_roads_there_s_a_drone_for_that}. [Accessed: 5- May- %2015]. 

%\bibitem{CNN13}
%CNN’s website, retrieved on July 29, 2015 from http://www.cnn.com/2013/12/02/tech/innovation/amazon-drones-questions/
%[Nol10] Nolan, Michael. Fundamentals of air traffic control. Cengage Learning, 2010.

%\end{thebibliography}

\newpage

\begin{IEEEbiography}[{\includegraphics[width=1in,height=1.25in,clip,keepaspectratio]{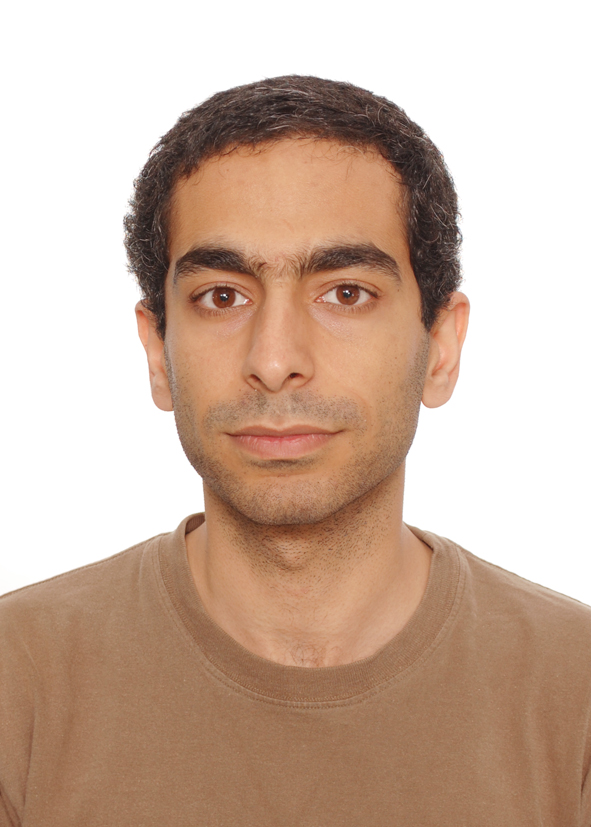}}]{Mirmojtaba Gharibi}
%\begin{IEEEbiography}{Mirmojtaba Gharibi}
received the B.ASc. in electrical engineering from Sharif University of Technology in 2009. He completed his M.Math, and currently is pursuing a PhD degree both in Computer Science at University of Waterloo, Waterloo, Canada.
\end{IEEEbiography}

\begin{IEEEbiography}[{\includegraphics[width=1in,height=1.25in,clip,keepaspectratio]{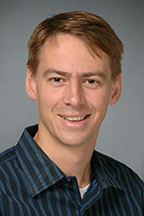}}]{Steven L. Waslander} is an Associate Professor with the University of
Toronto Institute for Aerospace Studies.  He received the B.Sc.E. degree
from Queen's University in 1998, and the M.S. and Ph.D. degrees from
Stanford University in Aeronautics and Astronautics, in 2002 and 2007,
respectively.  He joined the University of Waterloo as faculty in 2008
and moved to the University of Toronto, in 2018, where he founded and
directs the Toronto Robotics and Artificial Intelligence Laboratory. 
His current research interests include quadrotors, autonomous driving,
localization and mapping, object detection and tracking, and multirobot
coordination.
\end{IEEEbiography}

\begin{IEEEbiography}[{\includegraphics[width=1in,height=1.25in,clip,keepaspectratio]{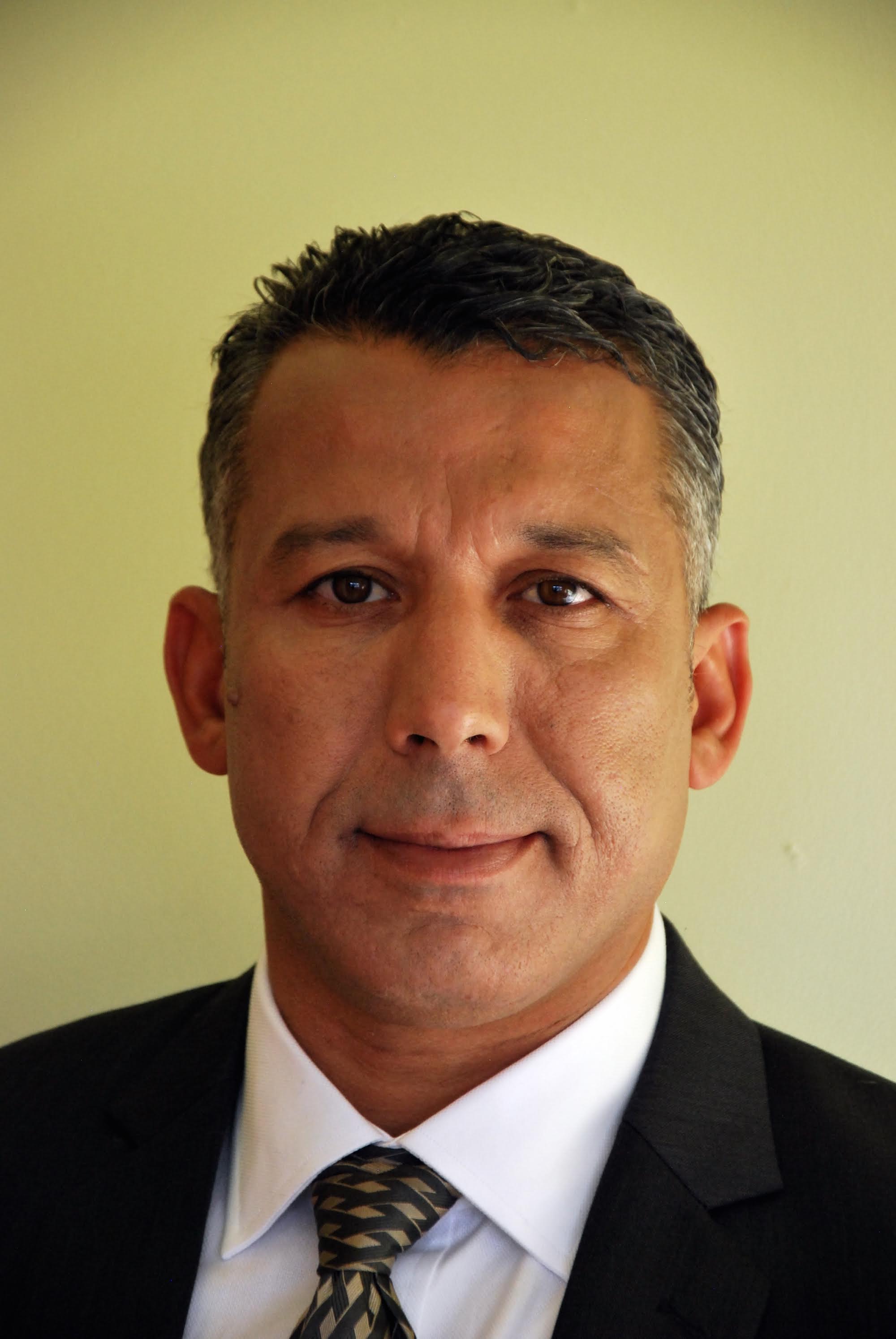}}]{Raouf Boutaba} received the M.Sc. (1990) and Ph.D (1994).
degrees in computer science from the University Pierre \& Marie Curie, Paris.
He is currently a university chair professor of computer science
at the University of Waterloo (Canada). His research interests include
resource management in wired and wireless networks. He was the
founding editor in chief of the IEEE transactions on Network and
Service Management (2007-2010) and the current editor in chief of the
IEEE journal on selected areas in communications. He is a fellow of the IEEE, the Engineering
Institute of Canada, and the Canadian Academy of Engineering.
\end{IEEEbiography}


%\vfill


\end{document}